\documentclass[10pt]{article} 
\usepackage[preprint]{rlc}

\usepackage{bbm}
\usepackage{amsmath}
\usepackage{amssymb}
\usepackage{mathtools}
\usepackage{amsthm}
\usepackage{amsfonts}
\usepackage{mathtools}          
\usepackage{mathrsfs}           
\mathtoolsset{showonlyrefs}     

\usepackage{graphicx}           
\usepackage{tikz, pgfplots} 
\usepackage{pgfplotstable}
\usepackage{subcaption}         
\usepackage{caption}
\usepackage[space]{grffile}     
\usepackage{url}                
\usepackage{pifont}

\usepackage{comment}
\usepackage{booktabs}  
\newcolumntype{M}[1]{>{\centering\arraybackslash}m{#1}}

\usetikzlibrary{arrows.meta}

\theoremstyle{plain}
\newtheorem{theorem}{Theorem}[section]
\newtheorem{proposition}[theorem]{Proposition}
\newtheorem{lemma}[theorem]{Lemma}
\newtheorem{corollary}[theorem]{Corollary}
\theoremstyle{definition}
\newtheorem{definition}[theorem]{Definition}

\theoremstyle{remark}

\DeclareMathOperator{\Tr}{Tr}

\newcommand{\EE}[2][]{\mathbb{E}_{#1}\left[#2\right]}

\newcommand{\Cov}[2][]{\mathrm{Cov}_{#1}\left[#2\right]}

\newcommand{\tr}[1]{\Tr\left(#1\right)}
\newcommand{\vect}[1]{{\rm vec}\left(#1\right)}
\newcommand{\Diag}[1]{{\rm Diag}\left(#1\right)}

\definecolor{ForestGreen}{rgb}{0,0.27,0.13}
\newcommand{\cmark}{{\color{ForestGreen}\ding{51}}}
\newcommand{\xmark}{{\color{red}\ding{55}}}

\title{ On the Limited Representational Power of Value Functions and its Links to Statistical (In)Efficiency}

\author{David Cheikhi  \\
    d.cheikhi@columbia.edu \\
    Graduate School of Business\\
    Columbia University
    \And
    Daniel Russo \\
    Graduate School of Business\\
    Columbia University}

\begin{document}

\maketitle

\begin{abstract}

Identifying the trade-offs between model-based and model-free methods is a central question in reinforcement learning. Value-based methods offer substantial computational advantages and are sometimes just as statistically efficient as model-based methods. However, focusing on the core problem of policy evaluation, we show information about the transition dynamics may be impossible to represent in the space of value functions. We explore this through a series of case studies focused on structures that arises in many important problems. In several, there is no information loss and value-based methods are as statistically efficient as model based ones. In other closely-related examples, information loss is severe and value-based methods are severely outperformed. A deeper investigation points to the limitations of the representational power as the driver of the inefficiency, as opposed to failure in algorithm design.

\end{abstract}

\section{Introduction}
\label{sec:introduction}

Policy evaluation is central to sequential decision making. Two main frameworks compete to tackle this problem. Model based methods start by learning the underlying dynamics of the problem (the model) and output a value-to-go estimate consistent with this learned model. While this approach enjoys nice theoretical properties (see e.g. Chapter 7 of \cite{moerland2023model}),  computing the value-to-go at a single state typically requires computationally expensive forward simulation  of the model. At the other end of the spectrum, model-free methods (or value-based) methods directly estimate the value function that fits data best, skipping the need for the difficult step of computing the value function corresponding to a model. As a result, direct 
 model-free estimation of value functions is ubiquitous in reinforcement learning. 

When they work well, model-free approaches to value function estimation can be statistically and computationally efficient while sidestepping the need to estimate details of an underlying model. Unfortunately, model-free algorithms are sometimes sample inefficient. 

\paragraph{Puzzling case-studies on the statistical efficiency of model-free methods.}
In studying this sample efficiency tradeoff in simple examples, we encounter puzzling competing evidence. Consider Figure  \ref{fig:teaser}. It compares the error of least-squares temporal difference learning \citep{bradtke1996linear} against simple model-based least-squares baseline in three cases. Section \ref{sec:LDS} provides further details. Subfigure \ref{fig:teaser_quad} looks at a linear quadratic control. As in \cite{tu2018least}, we see LSTD is severely outperformed by model-based. \cite{recht2019tour} conjectures that the issue is that model-based algorithms leverage vector state observations, whereas LSTD collapses these into scalar observations (of reward plus value-to-go). Subfigure \ref{fig:teaser_lin} raises doubt about this interpretation. LSTD is efficient again if we adjust the linear-quadratic setting so that rewards are linear, rather than quadratic, functions of state. Why is the sample-efficiency so impacted by this change to the problem?

In Subfigure \ref{fig:teaser_diag}, we restrict the model dynamics to be diagonal. This structure is motivated by problems with decoupled structure --- which we study in detail in Section \ref{sec:DMRP} and we explain is a very important feature of many practical problems. We see that model-based are again far superior.

\captionsetup[subfigure]{format=hang}
\begin{figure}[h!]
\centering
\captionsetup[subfigure]{justification=centering}
{\begin{subfigure}[b]{0.35\columnwidth}
\centering
\begin{tikzpicture}
\begin{axis}[
            title={},
            xmin=0,xmax=22,
            ymin=0,ymax=2300,
            width=5cm,
            height=5cm,
            table/col sep=comma,
            xlabel = Dimension $d$,
            ylabel = MSE $\EE{\|\hat{\beta} -\beta \|^2}$,
            grid=both,
            legend pos = north west]

\addplot [black,mark=*,only marks, mark options={scale=0.5}, error bars/.cd, y explicit,y dir=both] table[x=d,y=Model-free_empirical, y error plus expr=\thisrow{CI_MF_UB} - \thisrow{Model-free_empirical} , y error minus expr=\thisrow{Model-free_empirical} - \thisrow{CI_MF_LB}] {Data/quad_experiment.csv};
\addlegendentry{LSTD};
\addplot [red,mark=+,only marks, mark options={scale=0.5}, error bars/.cd, y explicit,y dir=both] table[x=d,y=Model-based_empirical, y error plus expr=\thisrow{CI_MB_UB} - \thisrow{Model-based_empirical} , y error minus expr=\thisrow{Model-based_empirical} - \thisrow{CI_MB_LB}] {Data/quad_experiment.csv};
\addlegendentry{Model Based};
\end{axis}
\end{tikzpicture}
\caption{General dynamics \\ Quadratic rewards\label{fig:teaser_quad}}
\end{subfigure}}
{\begin{subfigure}[b]{0.28\columnwidth}
\centering
\begin{tikzpicture}
\begin{axis}[
            title={},
            xmin=0,xmax=55,
            ymin=0,ymax=340,
            width=5cm,
            height=5cm,
            table/col sep=comma,
            xlabel = Dimension $d$,
            grid=both,
            legend pos = north west]

\addplot [black,mark=*,only marks, mark options={scale=0.5}, error bars/.cd, y explicit,y dir=both] table[x=d,y=Model-free_empirical, y error plus expr=\thisrow{CI_MF_UB} - \thisrow{Model-free_empirical} , y error minus expr=\thisrow{Model-free_empirical} - \thisrow{CI_MF_LB}] {Data/lin_experiment.csv};
\addlegendentry{LSTD};
\addplot [red,mark=+,only marks, mark options={scale=0.5}, error bars/.cd, y explicit,y dir=both] table[x=d,y=Model-based_empirical, y error plus expr=\thisrow{CI_MB_UB} - \thisrow{Model-based_empirical} , y error minus expr=\thisrow{Model-based_empirical} - \thisrow{CI_MB_LB}] {Data/lin_experiment.csv};
\addlegendentry{Model Based};

\end{axis}
\end{tikzpicture}
\caption{General dynamics\\ Linear rewards\label{fig:teaser_lin}}
\end{subfigure}}
{\begin{subfigure}[b]{0.28\columnwidth}
\centering
\begin{tikzpicture}
\begin{axis}[
            title={},
            xmin=0,xmax=55,
            ymin=0,ymax=340,
            width=5cm,
            height=5cm,
            table/col sep=comma,
            xlabel = Dimension $d$,
            grid=both,
            legend pos = north west]

\addplot [black,mark=*,only marks, mark options={scale=0.5}, error bars/.cd, y explicit,y dir=both] table[x=d,y=Model-free_empirical, y error plus expr=\thisrow{CI_MF_UB} - \thisrow{Model-free_empirical} , y error minus expr=\thisrow{Model-free_empirical} - \thisrow{CI_MF_LB}] {Data/linear_diagonal.csv};
\addlegendentry{LSTD};
\addplot [red,mark=+,only marks, mark options={scale=0.5}, error bars/.cd, y explicit,y dir=both] table[x=d,y=Model-based_empirical, y error plus expr=\thisrow{CI_MB_UB} - \thisrow{Model-based_empirical} , y error minus expr=\thisrow{Model-based_empirical} - \thisrow{CI_MB_LB}] {Data/linear_diagonal.csv};
\addlegendentry{Model Based};

\end{axis}
\end{tikzpicture}
\caption{Diagonal dynamics\\ Linear rewards\label{fig:teaser_diag}}
\end{subfigure}}
\caption{Mean-squared error of LSTD and model based estimators when dynamics are linear and rewards are quadratic (\ref{fig:teaser_quad}), when rewards are simplified to be linear (\ref{fig:teaser_lin}) and when dynamics are further simplified to be diagonal (\ref{fig:teaser_diag}).\label{fig:teaser}}
\end{figure}
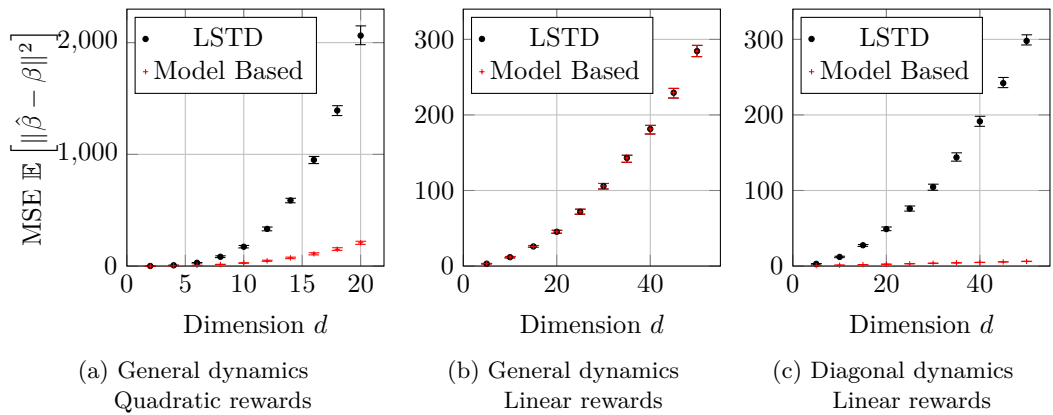

All scenarios may initially seem quite favorable to model-free methods, since the structure of the model implies a simple structure to the true value function. Yet (batch) LSTD is grossly sample inefficient in two cases; in one other, it is just as efficient as model-based methods. 

\paragraph{Limited representational power of value functions.}
 We offer an explanation of the puzzling scenarios plotted above. The issue is not with LSTD itself, but with a loss of information when translating a simplifying structure of the model dynamics to an associated (simplifying) structure of the value-function representation. 
\begin{quote}
    \emph{Model free estimation algorithms suffer inherent information loss when  information about the transition dynamics cannot be encoded in a value function representation. } 
\end{quote}
To  understand this precisely, it is helpful to abstract away from the details of LSTD. We view it as some rule for estimating value functions based on data. When specific problem structures are known, its behavior can be improved by supplying a specialized value function representation, e.g. by taking $\mathbb{V}$ below to consist of quadratic value functions.

\begin{definition}\label{def:modelfree}
    A model-free value function estimation algorithm (e.g. LSTD) is a function that maps a class of value functions $\mathbb{V}$ and a dataset $D= \{(s_i,r_i,s'_i) : i=1,\ldots, n\}$ consisting of state, reward, next-state tuples, to an estimated value function $\hat{V} \in \mathbb{V}$.
\end{definition}

The next definition considers the conversion from a class of MRP models $\mathbb{M}$ to the induced class of value functions $\mathbb{V} =\{ V_{M}: M \in \mathbb{M}\}$ where $V_M$ is the value function associated with model $M$. We say that the value function representations lose information about the model structure if $\mathbb{V}$ is also induced by a larger class of models $\mathfrak{M}$. Roughly, a generic algorithm like LSTD cannot `know' or exploit a structure of the class of models $\mathbb{M}$ that is not present in $\mathfrak{M}$ .

\begin{definition}[Loss of information]\label{def:info-loss}
    For a class of models $\mathbb{M}$, we say there is a loss of information in the space of value functions if there exists a larger set of models $\mathfrak{M}$, that is $\mathbb{M} \subsetneq \mathfrak{M}$, such that $\mathbb{M}$ and $\mathfrak{M}$ have the same class of value functions:
    \begin{equation*}
        \{V_M; M \in \mathbb{M} \} = \{V_M; M \in \mathfrak{M}\}.
    \end{equation*}
\end{definition}

A deeper look into the examples in Figure \ref{fig:teaser} seems to confirm that information-loss is the true driver of LSTD's poor performance. In fact, in both the diagonal dynamics case in Figure \ref{fig:teaser_diag} and the linear-quadratic example in Figure \ref{fig:teaser_quad}, we explain that LSTD is  \emph{equivalent} to a model-based least-squares estimation procedure that operates in a much larger space of models that shares the same value-function representation, like $\mathfrak{M}$ in Definition \ref{def:info-loss}. Therefore, our theory offers an alternative explanation to the findings of \cite{tu2019gap}.

As depicted in Table \ref{table:summary}, the presence of severe information loss predicts the relative statistical efficiency or inefficiency of model-free methods in five important examples. In addition to the linear problems in Figure \ref{fig:teaser}, we also study policy evaluation in problems with many state components that evolve independently. This structure is ubiquitous in operations research, see e.g. Chapter 6.2 of \cite{Bert05}, and also appears naturally in other domains like recommender systems \citep{maystre2023optimizing}. The very poor performance of generic model-free algorithms on simple variants of such problems is disconcerting.
 
\begin{table}[ht]
\centering
\begin{tabular}{|M{4cm}|M{2.5cm}|M{4cm}|M{1cm}|M{1.5cm}|}
\hline
 Model class & Value functions consistent with model class& Models consistent with  value function class &No info. loss& LSTD is data efficient\\
\hline
All MRPs& All functions&  All MRPs  & \cmark & \cmark \\ \hline 
Decoupled transitions; additive rewards& Linearly separable functions& MRPs with arbitrary transition structure & \xmark &  \xmark \\ \hline 
\emph{Decoupled} linear dynamics; linear rewards& Linear functions& MRPs with \emph{general}  linear dynamics and rewards & \xmark & \xmark \\ \hline 
General linear dynamics; linear rewards& Linear functions& (essentially) MRPs with general linear  transition structure& \cmark & \cmark\\ \hline 
 LQR: general linear dynamics; \emph{quadratic} rewards  & Quadratic functions& A class of MRPs with $\Theta(n^4)$ free parameters (instead of $\Theta(n^2)$). & \xmark & \xmark\\ \hline
\end{tabular}
\caption{In the settings studied throughout this paper, there is a perfect correlation between LSTD's relative statistical inefficiency and whether there is los of information in the space of value functions.}
\label{table:summary}
\end{table}

It is worth stressing that all of our examples are purposefully stylized and meant to serve as a sanity check on the performance of model-free value function estimation procedures.  It is usually easy to modify algorithms to perform well on a particular example, but it is unfortunate that these case-by-case modifications are necessary.

\section{Related work}
Results about sample efficiency (or inefficiency) of both model-free \citep{jin2020provably, chen2021infinite} and model-based \citep{kearns2002near, auer2008near, azar2017minimax} are numerous in the literature. In specific settings, the statistical efficiency of model-free and model-based methods have been compared, showing there is nearly no gap between the two frameworks in the tabular case \citep{zanette2019tighter, jin2018q} but that such a gap exists in the LQR setting \citep{tu2019gap}. The results from these works are setting-specific and provide little intuition on what characteristic of the problem at hand drives the performance discrepancies in Figure \ref{fig:teaser}. 

\cite{sun2019model} studies this issue in the full RL problem, where good performance requires exploration and planning and not just policy evaluation. They show there exist a setting in which any model-free method requires exponentially more samples to match the performance of a model-based method. While superficially similar, our paper offers insights that are very different and complementary. Our definition of model-free policy evaluation methods (Def~\ref{def:modelfree}) restricts what problem structures are known apriori to the learner, leading to our emphasis on information-loss due to value-based representation. Roughly speaking, we're worried that a sound value-function based representation is necessarily ``too flexible'' to reflect structural constraints on the transition dynamics. 
The definition in \cite{sun2019model} instead restricts what model-free methods can observe about the dataset itself if the value-function class is \emph{insufficiently flexible}. In particular, their model-free methods only have access to the image of states under value functions in some class. Their counter-example constructs a setting where important features of the data are lost as a result. This definition of model-free methods is interesting, but  is vacuous  in the settings discussed in Table \ref{table:summary},  since for the data could be perfectly recovered from sufficiently many value function evaluations.

\cite{zhu2023representation} and \cite{dong2020expressivity} exhibit settings where value functions require very complex representations even though the the transition dynamics are simple. They argue this makes model-free methods inefficient. This is orthogonal to the limitations in representation power described in this work which refers to what information about the model space can be inferred from the sole knowledge of the space of value functions. All of our counter-examples stress the potential inefficiency of value-based methods even when the value-function has known and simple representation. 

Other works study benefits or shortcomings of both types of methods, such as robustness to misspecification \citep{clavera2018model, jin2020provably, wang2020provably, zhu2023provably} or safety \citep{berkenkamp2017safe, xu2023uniformly}.

Attempting to reconcile the benefits of model-free methods with the statistical efficiency of model-based methods, multiple works have explored using a learned model as a generator to enhance statistical efficiency of model-free methods, both theoretically \citep{feinberg2018model, buckman2018sample, janner2019trust} and empirically \citep{lu2024synthetic, young2022benefits}.

\section{Problem formulation}
\label{subsec:VFE}

A \emph{trajectory} in a Markov Reward Process (MRP) is a Markovian sequence $\tau=(S^{(0)}, R^{(1)}, S^{(1)}, R^{(2)}, S^{(2)},  \dots),$ consisting of a sequence of states $(S^{(t)})_{t} \subset \mathcal{S}$ and rewards $(R^{(t)})_{t}\subset \mathbb{R}$. MRPs are the common framework for policy evaluation: evaluating a fixed policy in a Markov Decision Process can be cast as evaluating the value function of a MRP. The law of a Markov Reward Process is specified by the tuple $\mathcal{M} = (\mathcal{S}, P, R)$ consisting of a state space, a transition kernel and a reward distribution. Here $P$ is a transition kernel over the state space $\mathcal{S}$, specifying a probability $P(s' \mid s)$ of transitioning from $s$ to $s'$. The object $R$ specifies the distribution of rewards conditioned on a state transition as $R(dr | s,s') = \mathbb{P}(R^{(t)} = dr \mid S^{(t)}=s, S^{(t+1)}=s')$. Throughout we use the notation $r(s,s')$ for the mean of $R(\cdot | s,s')$. For a given discount factor $0 < \gamma \leq 1$, the value function
\begin{align*}
V(s) = \mathbb{E}\left[ \sum_{t=1}^{\infty} \gamma^t R^{(t)} \mid S^{(0)}=s \right] = \mathbb{E}\left[ \sum_{t=1}^{\infty}  \gamma^t r\left(S^{(t)}, S^{(t+1)}\right) \mid S^{(0)}=s \right]
\end{align*}
 specifies the expected sum of discounted rewards along a trajectory. We focus on the problem of estimating the value function $V$ given a dataset of transitions $\mathcal{D} = \left\{ \left((S^{(t)}, R^{(t)}, S^{(t)}_{\rm next} \right) \right\}_{t=1, \dots, n}$ where $S^{(t)}_{\rm next} \sim P(\cdot | S^{(t)})$ and $R^{(t)} \sim R(\cdot | S^{(t)}, S^{(t+1)})$. This allows for the case where $S^{(t)}$ are generated from a single trajectory (in which case $S^{(t+1)} = S^{(t)}_{\rm next}$), from multiple trajectories or drawn i.i.d. from the occupancy measure.

\section{Markov Reward Processes with Decoupled Transition Structure}
\label{sec:DMRP}

\subsection{On the importance of decoupled structures}

The curse of dimensionality is a major barrier in Reinforcement Learning as the number of parameters that need to be learned typically scales exponentially with the dimension of the state space. Thankfully, it is common that all parts of the state cannot interact arbitrarily in practical problems. 
For instance, \cite{maystre2023optimizing} model  recommender systems where a component of the user's state represents their habitual engagement  user with specific audio items (e.g podcasts). The presence of hundreds of thousands of items means state variables have hundreds of thousands of components. Thanfuklly recommending a new podcast might primarily affect the few components related to the user's listening habit with that podcast.

Given the importance and omnipresence of these problems in practice, a large literature has studied problems having some form of decoupling or weak coupling. Factored Markov Decision Processes \citep{kearns1999efficient} or weakly coupled Markov Decision Processes \citep{hawkins2003langrangian} are two common frameworks for capturing such structure.  
Examples of real-life problems cast in these frameworks include inventory management \citep{turken2012multi, chen2021primal}, queuing \citep{chen2021primal}, resource allocation \citep{tesauro2005online}, or restless bandits \citep{weber1990index},  personalized health interventions \citep{baek2023policy}, and many others. 

In what follows, we restrict our study to the most elementary version of this structure where all components are perfectly independent and the rewards can be linearly decomposed. Even though some weakly coupled problems can be tackled by solving a series of fully decoupled problems as proposed by \cite{le2019batch}, we do not claim to exactly capture practical problems through this construction. Instead, we view decoupled MRPs as a minimal example capturing the essence of those problems. That LSTD fails to be statistically efficient even in this case is noteworthy. Given the simplicity of our example, we expect that the exhibited failures will continue or be made worse in more complex settings.

\subsection{Definition}
A $d$-decoupled MRP $\mathcal{M}$ is a system where the state and reward are respectively the concatenation of the states and the sum of rewards of $d$ MRPs $\mathcal{M}_i  = (\mathcal{S}_i, P_i, R_i)$ evolving independently.
\begin{definition}[Decoupled MRP] 
    A MRP $\mathcal{M} = (\mathcal{S}, P, R)$ is said to be $d-decoupled$ if
    \begin{itemize}
        \item There exists $\mathcal{S}_1, \dots, \mathcal{S}_d$ such that $\mathcal{S} = \mathcal{S}_1 \times \dots \times \mathcal{S}_d$.
        \item There exists transition kernels $P_1, \dots P_d$  and reward kernels $R_1, \dots R_d$  such that for all states $s, s'$, $P(s' | s )  = \prod_{i=1}^d P_i (s_i' | s_i)$ and $r(s, s')  = \sum_{i = 1}^d r_i (s_i , s_i')$.
    \end{itemize}
\end{definition}

 In the rest of this section, we restrict ourselves to the case of tabular decoupled MRPs, that is each $\mathcal{S}_i$ is finite. The following property of decoupled MRPs is central to the use of value-based methods in this setting: the value function of a decoupled MRP is linearly separable, that is there exists $V_1, \dots, V_d$ such that
\begin{equation}\label{eq:separable_v}
    V(s) = \sum_{i=1}^d V_i(s_i) \quad \text{for all } s\in \mathcal{S}.
\end{equation}
Let $\mathbb{M}_D$ denote the set of decoupled MRPs and $\mathbb{V}_D$ denote the set of separable value functions, that is the value functions $V$ that verify \eqref{eq:separable_v} for some $V_1, \dots, V_d$. Then $\mathbb{V}_D$ is exactly the class of value functions induced by decoupled MRPs.

\begin{proposition} \label{prop:dec_is_sep}
    The class of value functions of decoupled MRPs is the set of separable value functions: $\{V_M | M \in \mathbb{M}_D \} = \mathbb{V}_D$.
\end{proposition}
\vspace{-0.2in}
\begin{proof} \boldmath $\{V_M | M \in \mathbb{M}_D \} \subset \mathbb{V}_D$ \unboldmath: Let $\mathcal{M}$ be a decoupled MRP, concatenation of $\mathcal{M}_1, \dots, \mathcal{M}_d$ with value functions $V_1 \dots, V_d$, respectively. Then $V(s) = \sum_{i=1}^d V_i(s_i)$ for all  $s \in \mathcal{S}$.

\boldmath $\mathbb{V}_D \subset \{V_M | M \in \mathbb{M}_D \}$\unboldmath: Let $V \in \mathbb{V}_D$ and $V_1, \dots, V_d$ be such that \eqref{eq:separable_v} is verified. There exists MRPs $\left(\mathcal{M}_i = (\mathcal{S}_i, P_i, R_i)\right)_{i=1,\dots, d}$ with value function $V_i$. The decoupled MRP obtained by concatenating $\left(\mathcal{M}_i\right)_{i = 1, \dots, d}$ has value function $V$.
\end{proof}

\subsection{Decoupled structure cannot be encoded in the value space}

Proposition \ref{prop:dec_is_sep} shows that the class of separable value functions is \textbf{exactly} the right class of value functions for decoupled MRPs. However, this is not sufficient for the decoupled structure of the transitions to be encoded in the space of value functions, as the following result shows.

\begin{theorem}[Value representation loses information]     \label{thm:many_sep} 
$\left(\mathbb{M}_D \subsetneq \{ M : V_M \in \mathbb{V}_D \}\right)$ \\
For any separable value function $V \in \mathbb{V}_D$ and any transition kernel $P \in \mathbb{R}^{n^d \times n^d}$ over states $\mathcal{S}$, there exists a vector of average rewards $r$ such that the value function of any MRP with transition kernel $P$ and average rewards $r$ is separable.
\end{theorem}
\begin{proof}
  Let $V \in \mathbb{V}_D$ and let $P \in \mathbb{R}^{n^d \times n^d}$ be any transition kernel. Setting $r = (I-\gamma P)V$ ensures that the value function of the MRP with transition kernel $P$ and linear reward $r$ is exactly $V$.
\end{proof}
Theorem \ref{thm:many_sep} shows that even if the set of value functions has been narrowed down as much as possible, this same class of value function is shared by another class of models whose dimension is exponentially larger. 

\subsection{Statistical inefficiency of LSTD}
\label{sec:par_sec}

We empirically demonstrate the inefficiency of LSTD on this type of problem, we consider $d$ tabular MRP $\mathcal{M}_i$ with $\mathcal{S}_i = \{1, \dots, N\}$. We then create a sequence of decoupled MRPs $(\tilde{\mathcal{M}}_k)_k$ of increasing dimension $k = 1, \dots, d$ by having $\tilde{S}^{(t)}_k = (S_1^{(t)}, \dots, S_{k}^{(t)})$ and $\tilde{R}_k^{(t)} \sim \mathcal{N}\left( \sum_{i=1}^{k} r_i\left(S_i^{(t)}\right), 1 \right)$ where $S_i^{(t)}$ is the state at time $t$ in $\mathcal{M}_i$ and $r_i(s)$ is the expected reward when at state $s$ in $\mathcal{M}_i$. 

We consider the ratio between the Mean Squared Error (MSE) $\EE{\|\hat{V}(s) - V(s)\|^2}$ under model-free and model-based estimates. Figure \ref{fig:par_exp} shows the evolution of this ratio as the number of components $k$ grows. Those experiments show that the gap in number of samples required by model-based and LSTD grows rapidly with the number of components. Subfigure \ref{fig:offline_par} displays the ratio when the expectation is taken uniformly over $\mathcal{S}$ to capture off-policy performance and Subfigure \ref{fig:online_par} when the expectation is taken uniformly over the dataset to capture on-policy performance.

\begin{figure}[h!]
\centering
\begin{subfigure}[b]{0.48\columnwidth} 
\begin{tikzpicture}
\begin{axis}[
            title={},
            xmin=0,xmax=210,
            ymin=0,ymax=30,
            width=5.5cm,
            height=5.5cm,
            table/col sep=comma,
            xlabel = Dimension $d$,
            grid=both,
            ylabel = Ratio of MSE $\frac{\EE{\|\hat{V}_{\rm LSTD}(s) - V(s) \|^2}}{\EE{\|\hat{V}_{\rm MB}(s) -V(s)\|^2}}$,
            legend pos = north west]

\addplot [black,mark=*,only marks, mark options={scale=0.5}, error bars/.cd, y explicit,y dir=both] table[x=d,y=Ratio_empirical, y error plus expr=\thisrow{CI_ratios_UB} - \thisrow{Ratio_empirical} , y error minus expr=\thisrow{Ratio_empirical} - \thisrow{CI_ratios_LB}] {Data/exp_par_offline.csv};
\addlegendentry{Empirical ratio};
\end{axis}
\end{tikzpicture}
\caption{$s$ is chosen uniformly at random over $\mathcal{S}$\label{fig:offline_par}}
\end{subfigure}
\begin{subfigure}[b]{0.48\columnwidth} 
\begin{tikzpicture}
\begin{axis}[
            title={},
            xmin=0,xmax=210,
            ymin=0,ymax=15,
            width=5.5cm,
            height=5.5cm,
            table/col sep=comma,
            xlabel = Dimension $d$,
            grid=both,
            legend pos = north west]

\addplot [black,mark=*,only marks, mark options={scale=0.5}, error bars/.cd, y explicit,y dir=both] table[x=d,y=Ratio_empirical, y error plus expr=\thisrow{CI_ratios_UB} - \thisrow{Ratio_empirical} , y error minus expr=\thisrow{Ratio_empirical} - \thisrow{CI_ratios_LB}] {Data/exp_par_online.csv};
\addlegendentry{Empirical ratio};
\end{axis}
\end{tikzpicture}
\caption{$s$ is chosen uniformly at random in the dataset\label{fig:online_par}}
\end{subfigure}
\caption{Ratio of the mean-squared error estimation of LSTD and model based for a randomly generated decoupled MRP, using $\gamma = 0.9, N =5$. Here each sample $\hat{\beta}$ was obtained using a trajectory of length $n = 1000$ and $80$ such samples were averaged to obtain an estimation of the MSE.}
\label{fig:par_exp}
\end{figure}
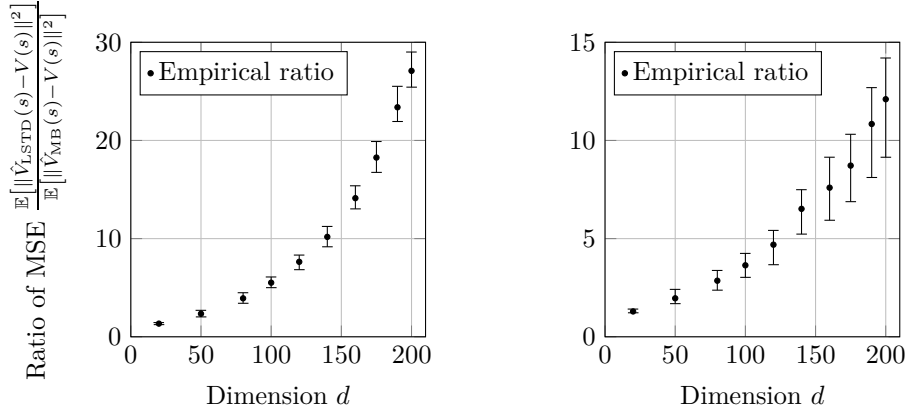

It is reasonable to think that the impossibility to represent the decoupled structure in the value space, as described in Theorem \ref{thm:many_sep}, is the root of this inefficiency. In the next section, we formalize this insight by considering a setting where LSTD can be viewed as doing model-based estimation on an larger than necessary class of models. 
                                                                                                                
\section{Linear dynamical systems}
\label{sec:LDS}

A linear dynamical system is one where there exists a stable matrix $A$ such that the transition kernel is of the form $X^{(t+1)} = A X^{(t)} + \epsilon^{(t)}$ and rewards are of the form $R^{(t)} = f_\theta(X^{(t)}) + \eta^{(t)}$, with $\epsilon^{(t)} \sim \mathcal{N}(0, \sigma^2 I_d)$ and $\eta^{(t)} \sim \mathcal{N}(0,\sigma^2)$. A model is characterized by its parameters $(A, \theta)$.

In what follows, we investigate three such systems which are variations of one another:
\begin{enumerate}
    \item General linear dynamics, linear rewards: In this setting, the dynamic matrix $A$ can be any stable matrix and the reward is linear $f_\theta(x) = \theta^\top x$. We call this class of models $\mathbb{M}_G = \{(A, \theta | A \in \mathbb{R}^{d \times d}, \theta \in \mathbb{R}^d \}$.
    \item Diagonal linear dynamics: The dynamic matrix $A$ is now constrained to be diagonal and the reward is again linear $f_\theta(x) = \theta^\top x$. We call this class of models $\mathbb{M}_D = \{(\Diag{a}, \theta | a \in \mathbb{R}^{d}, \theta \in \mathbb{R}^d \}$.
    \item Linear Quadratic Control (LQR): The dynamic matrix $A$ can be any stable matrix and the reward is quadratic: $f_Q(x) = x^\top Q x$. We call this class of models $\mathbb{M}_Q = \{(A, Q | A \in \mathbb{R}^{d \times d}, Q \in \mathbb{R}^{d \times d} \}$. Note that contrarily to the classic LQR, $Q$ isn't restricted to be SDP. While this hypothesis is important for control, it is not for the policy evaluation problem where only the symmetric structure is typically leveraged. While we don't assume $Q$ to be symmetric, our results easily extend to this case.
\end{enumerate}

As observed in Figure \ref{fig:teaser}, these subtle differences in the class of model candidates can lead to large variations in the performance of LSTD. In what follows, we show that in the first case (general linear dynamics and linear rewards), LSTD is equivalent to model based. However, in the two other cases, we show there is a significant loss of information in the value function space and LSTD can be viewed as a model based procedure in a larger set of models. 

The model-based estimators used in what follows are the plug-in estimator where the dynamics and reward functions are the least-square estimate. \cite{tu2018least} provides a description of LSTD both in the general case and specified to the LQR setting. The formal definition of all the estimators mentioned thereafter can be found in Appendix \ref{sec:DLS_proofs} (linear rewards) and \ref{sec:LSTD_eq_LQR} (quadratic rewards).

\subsection{General linear dynamics, linear rewards}
The class of value functions associated with models in $\mathbb{M}_G$ is the set of all linear functions: for a model $\mathcal{M} = (A, \theta)$, the associated value function is $V_\mathcal{M}(x) = \beta^\top x$ where $\beta = \left(I_d - \gamma A^\top\right)^{-1}\theta $. When $A$ spans the set of all stable matrices and $\theta$ the set of all vectors, then $\beta$ spans the set of all vectors. 

Hence, \begin{equation*}
    \{V_\mathcal{M}; \mathcal{M} \in \mathbb{M}_G \} = \{x \mapsto \beta^\top x; \beta \in \mathbb{R}^d \}
\end{equation*}

In this setting, running LSTD over this class of value functions does not suffer any loss of statistical efficiency compared to model-based as we show these two estimators are the same.

\begin{theorem}
    \label{thm:equivalence_LS}
    In the case of linear systems with linear costs, LSTD is equivalent to the model based estimator with no assumptions on the dynamics, that is $\hat{\beta}_{\rm LSTD} = \hat{\beta}_{\rm Model Based}$
\end{theorem}

This result can be observed by simply considering the closed form expression of both estimators. The detailed proof is defered to the Appendix \ref{subsec:proof_equivalence_LS}.

\subsection{Diagonal linear dynamics, linear rewards}\label{subsec:diagonal}

Restricting the dynamics to be diagonal drastically reduces the number of candidate models. Intuitively, leveraging this added structure should be necessary to reach statistical efficiency. However, the class of value functions doesn't shrink compared to the case where no restriction is imposed on the dynamics: 
\begin{equation*}
    \{V_\mathcal{M}; \mathcal{M} \in \mathbb{M}_D \} = \{x \mapsto \beta^\top x; \beta \in \mathbb{R}^d \} = \{V_\mathcal{M}; \mathcal{M} \in \mathbb{M}_G \}
\end{equation*}

Therefore, there is information loss when moving to the space of value functions. Since LSTD only leverages the knowledge of the class of value functions, it will output the same estimate whether it is known that the dynamics are diagonal or not. 
That is, in this setting, LSTD is equivalent to a model-based approach in the wrong class of models. It is, implicitly, learning $\Theta(d^2)$ free-parameters instead of $\Theta(d)$. The following theorem shows that it indeed leads in a loss of statistical efficiency. We also verify empirically this result, as shown in Figure \ref{fig:lin_ex}.

\begin{theorem}
\label{thm:DLS_gap}
    Let's consider the linear system with diagonal transitions $ A =\lambda I_d$ and linear rewards $\theta =\boldsymbol{1}_d $. Then,
    \begin{equation*}
       \lim_{n \to \infty} \dfrac{\EE{\|\hat{\beta}_{\rm LSTD} - \beta\|^2}}{\EE{\|\hat{\beta}_{\rm diag} - \beta\|^2}} = \Theta(d)
    \end{equation*}
\end{theorem}
\begin{proof}
    Using the fact that $\hat{\beta}_{\rm LSTD} = \hat{\beta}_{\rm MB}$ and Corollary \ref{cor:MB_estim}, we have
    \begin{align*}
        \lim_{n \to \infty} n\EE{\|\hat{\beta}_{\rm LSTD} - \beta\|^2} &= \left(\frac{d\gamma^2}{(1-\gamma\lambda)^2} + 1\right) \left(d \frac{1-\lambda^2}{(1-\lambda \gamma)^2}\right) 
        = \Theta(d^2)
    \end{align*}
    Similarly, from \ref{cor:diag_MSE}, we have $\lim_{n \to \infty} n\EE{\|\hat{\beta}_{\rm diag} - \beta\|^2} = d \left(\frac{\gamma^2}{(1-\gamma\lambda)^2} +1 \right) \frac{1 - \lambda^2}{(1-\gamma \lambda)^2} = \Theta(d)$. Finally, taking the ratio gives  the result. 
    
\end{proof}

\begin{figure}[h!]
\centering
\begin{center}
\centering
\centering
\begin{tikzpicture}
\begin{axis}[
            title={},
            xmin=0,xmax=50,
            ymin=0,ymax=50,
            width=5.5cm,
            height=5.5cm,
            table/col sep=comma,
            xlabel = Dimension $d$,
            ylabel = Ratio in MSE $\frac{\EE{\|\hat{\beta}_{\rm LSTD} -\beta \|^2}}{\EE{\|\hat{\beta}_{\rm diag} -\beta\|^2}}$,
            grid=both,
            legend pos = north west]

\addplot [red,mark=*,only marks, mark options={scale=0.5}, error bars/.cd, y explicit,y dir=both] table[x=d,y=Ratio_empirical, y error plus expr=\thisrow{CI_ratios_UB} - \thisrow{Ratio_empirical} , y error minus expr=\thisrow{Ratio_empirical} - \thisrow{CI_ratios_LB}] {Data/linear_diagonal.csv};
\addlegendentry{Empirical};
\addplot [black] table[x=d,y=Ratio_theoretical] {Data/linear_diagonal.csv};
\addlegendentry{Asymptotic};

\end{axis}
\end{tikzpicture}

\end{center}
\caption{Ratio of the mean-squared error estimation of LSTD and diagonal model based as a functoin of the dimension $d$. The system considered is the one with dynamics $A = 0.9 I_d$ and rewards $\theta = \boldsymbol{1}_d$, using $\gamma = 0.9$. Here $\lambda = \gamma = 0.9$. $100$ samples of each $\hat{\beta}$ were obtained using $n = 1000$ transitions.}
\label{fig:lin_ex}
\end{figure}
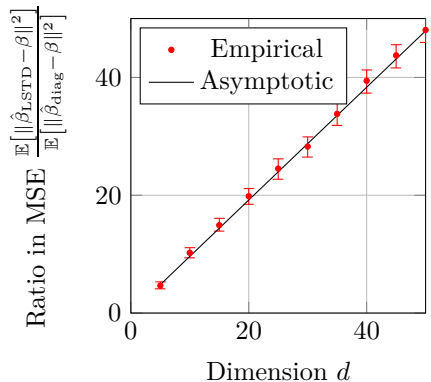

\subsection{Linear Quadratic Control}
\label{sec:LQR}
As already observed by \cite{tu2019gap}, LSTD fails to be statistically efficient in the LQR setting. We show that this setting also suffer a loss of information and LSTD is again equivalent to a model-based algorithm that optimizes over a larger class of models that induces this set of value function. The value function of a system with linear dynamics $A$ and quadratic rewards $Q$ is 
\begin{align*}
V_{A,Q}(x) = x^\top P x + \frac{\gamma}{1-\gamma} \tr{P} = \tr{(xx^\top + \frac{\gamma}{1-\gamma}I_d)P}
\end{align*}
where $P$ verifies the discrete Lyapunov equation $\left(\gamma^{1/2}A\right)^\top P \left(\gamma^{1/2}A\right) - P + Q = 0$. Once more, we show this the exactly the right class of value functions.

\begin{proposition}
    \label{prop:LQR_V_exact}
    For any $P \in \mathbb{R}^{d \times d}$, there exists a stable matrix $A$ and a reward matrix $Q$ such that 
    \begin{equation*}
        V_{A,Q}(x) = \tr{(xx^\top + \frac{\gamma}{1-\gamma}I_d)P}
    \end{equation*}
    That is
    \begin{equation}\label{eq:lqr_value_functions}
        \{V_\mathcal{M} | \mathcal{M} \in \mathbb{M}_Q\} =\left\{ x \mapsto \tr{(xx^\top + \frac{\gamma}{1-\gamma}I_d)P} | P \in \mathbb{R}^{d \times d} \right\}
    \end{equation}
\end{proposition}
\begin{proof}
    For a matrix $P$, choose $A = .5I_d$ and $Q = \frac{4}{4-\gamma}P$. 
\end{proof}
We now operate a change of variable, $Z^{(t)} = \vect{X^{(t)}(X^{(t)})^\top}$ where $\vect{M}$ is the vector obtained by stacking the columns of $M$. Since the mapping $x \mapsto xx^\top$ is injective, there is no loss of observability, that is observing $X^{(t)}$ or $Z^{(t)}$ is equivalent. 

In the space of $Z^{(t)}$, the system and rewards are affine:
\begin{align*}
    \EE{Z^{(t+1)} | Z^{(t)}} &= M Z^{(t)} + \vect{I_d}; \quad \quad \text{where }  M = A \otimes A
    \\\EE{R^{(t)} | Z^{(t)}} &=  \vect{Q}^\top Z^{(t)}
\end{align*}
where $\otimes$ is the Kronecker product. 

Viewing the problem as one with state variable $Z^{(t)}$  lets us draw a clarifying parallel to the setting where the rewards are linear and the dynamics are known to be diagonal (Subsection \ref{subsec:diagonal}): in both settings, the class of possible transitions is a subset of all possible linear dynamics. Instead of the diagonal structure, in this example we know the dynamics matrix $M$ is of the form $A \otimes A$ instead of being any arbitrary matrix in $\mathbb{R}^{d^2 \times d^2}$. Working solely with value function representations loses this structure on the dynamics matrix $M$; this is what we call \emph{information-loss}.

Given this discussion, it becomes unsurprising that LSTD is inefficient even when applied with the right class of value functions. In fact, similar to the setting with diagonal dynamics, LSTD is equivalent to a model-based method that estimates the dynamics matrix $M$ without incorporating known constraints. Information about those structural constraints has been ``lost'' when working with a value-basted representation. 
\begin{theorem}[Informal]
    LSTD applied with value functions in \eqref{eq:lqr_value_functions} is equivalent to a model-based plug-in estimator where the matrix $M\in \mathbb{R}^{d^2 \times d^2}$ is estimated by unconstrained least-squares. 
\end{theorem}
 A precise description of this model-based procedure can be found in Appendix \ref{sec:LSTD_eq_LQR}.

\section{Conclusion}

In this work, we show that some information about the model may not be representable in the space of value functions. By exploring multiple closely related settings where the efficiency of model-free methods varies significantly, we are able to correlate the (in)efficiency of model-free methods with whether there is loss of information in the space of value functions. We are able to establish this loss of information as the driver of statistical inefficiency in several of our examples by observing that LSTD implicitly operates as model-based methods over an excessively large class.   

Importantly, this work does not imply an impossibility of model-free methods to be efficient but instead hints that any ``generic''
algorithm would only be efficient for a few settings. When problem structure (e.g. decoupled structure) is known or learned from auxiliary tasks, efficient learning requires exploiting this structure. 
Information-loss suggests that value-based estimation procedures cannot adapt to some problem structures based on the value function representation alone, since the value function representation is invariant to that structure. Proving a formal version of this statement is left for future work.

It is worth emphasizing that if one had a specific structure in mind, like the decoupled MRP structure, then it may be possible to design value-based estimation procedures that are tailored to that structure and enhance efficiency. In this case, the modeler is overcoming information loss by  reflecting problem structure in the estimation algorithm and not the value function class. 

\appendix

\bibliography{main}
\bibliographystyle{rlc}

\section{Proof of Theorem \ref{thm:DLS_gap}} \label{sec:DLS_proofs}
Theorem \ref{thm:equivalence_LS} states that LSTD over the class of all linear functions is equivalent to the plug-in estimator in the case where nothing is assumed about the linear system. The proof will work by comparing the limiting variance of the model based estimator when the dynamics are known to be linear versus when nothing is known about the dynamics. We do so by stating a Central Limit Theorem in both cases. 

\subsection{Notations and preliminaries}

In what follows, $A \in \mathbb{R}^d$ is a stable matrix, $\theta \in \mathbb{R}^d$ is a reward vector. $(\epsilon^{(t)})_t$ and $(\eta^{(t)})_t$ are two independent i.i.d. sequences such that $\epsilon^{(t)} \sim \mathcal{N}\left(0, \sigma^2 I\right)$ and $\eta^{(t)} \sim \mathcal{N}(0, \sigma^2)$. The dataset accessed by any algorithm is of the form $(X^{(t)}, R^{(t)}, X^{(t+1)})_{t=1, \dots, n}$ where 
\begin{align*}
    X^{(t+1)} = A X^{(t)} + \epsilon^{(t)}; \quad \quad R^{(t)} = \theta^\top X^{(t)} + \eta^{(t)}
\end{align*}
$\beta$ is the parameter of the value function (that is $V(x) = \beta^\top x$) and verifies $\beta = (I_d - \gamma A^\top)^{-1}\theta$. The stationary covariance matrix $P_\infty$ of the process $(X^{(t)})_t$ exists \citep{mann1943statistical} and verifies 
$$A P_\infty A^\top - P_\infty + \sigma^2 I_d = 0.$$
Finally, $\otimes$ is the Kronecker product of two matrices and $\| \cdot \|$ is the Frobenius norm. We now formally introduce the estimators we are comparing.

\begin{definition}[LSTD]
The LSTD estimator when the systems has linear dynamics and linear rewards is 
\begin{equation*}
    \hat{\beta}_{\rm LSTD} = \left(\sum_{t=1}^n X^{(t)} (X^{(t)} - \gamma X^{(t+1)})^\top  \right)^{-1}\sum_{t=1}^n R^{(t)} X^{(t)}
\end{equation*}
\end{definition}

\begin{definition}[Model Based estimator]
The model based estimator considered in this work are of the form
\begin{equation*}
    \hat{\beta}_{\rm Model Based}(\mathcal{A}) = \left(I - \gamma \hat{A}^\top\right)^{-1}\hat{\theta}
\end{equation*}
    where 
    \begin{align*}
        \hat{A} = \arg\min_{A \in \mathcal{A}} \sum_{t = 1}^n \| X^{(t+1)} - A X^{(t)} \|^2; \quad \quad \hat{\theta} = \arg\min_{\theta \in \mathbb{R}} \sum_{t = 1}^n ( R^{(t)} - \theta^\top X^{(t)} )^2.
    \end{align*}
    When the class of models considered is not specified, it is understood there is no restriction $$\hat{\beta}_{\rm Model Based} = \hat{\beta}_{\rm Model Based}(\mathbb{R}^{d \times d}).$$ 
    In the case of diagonal transitions, we write $$\hat{\beta}_{\rm diag} = \hat{\beta}_{\rm Model Based}(\Diag{\mathbb{R}^{d}}).$$
\end{definition}

\subsection{Proof of Theorem \ref{thm:equivalence_LS}}
\label{subsec:proof_equivalence_LS}
To facilitate the proof, we rewrite the estimators of interest in their closed form in matrix notation. First LSTD is
\begin{equation*}
    \hat{\beta}_{\rm LSTD} = \left(X(X-\gamma X')^\top\right)^{-1}XR
\end{equation*}

where $X \in \mathbb{R}^{d \times n}$ is the matrix where the $t$-th column is $X^{(t)}$, $X' \in \mathbb{R}^{d \times n}$ is the matrix where the $t$-th column is $X^{(t+1)}$ and $R \in \mathbb{R}^n$ the vector where $R_t = R^{(t)}$.

On the other hand, the least-square estimate of the dymaic matrix $A$ and of the reward vector $\theta$ are
\begin{align*}
    \hat{A}^\top = (XX^\top)^{-1} X(X')^\top ; \quad \quad \hat{\theta} = (XX^\top)^{-1} XR
\end{align*}

This allows to re-write the plug-in estimator as
\begin{align*}
    \hat{\beta}_{\rm Model Based}(\mathbb{R}^{d \times d}) &= \left(I_d - \gamma A^\top\right)^{-1} \hat{\theta} 
    = \left(I_d - \gamma(XX^\top)^{-1} X(X')^\top \right)^{-1} (XX^\top)^{-1} XR\\
    &= \left(XX^\top - \gamma X(X')^\top\right)^{-1} XR = \left(X(X - \gamma X')^\top\right)^{-1} XR\\
    &= \hat{\beta}_{\rm LSTD}.
\end{align*}

\subsection{Central Limit Theorem in the general linear dynamic setting}

\begin{proposition}
    The LSTD estimator $\hat{\beta}_{\rm LSTD} \to \beta$ a.s. and 
    \begin{align*}
        \sqrt{n} \left(\left(I_d - \gamma \hat{A}^\top \right)^{-1} - \left(I_d - \gamma A^\top \right)^{-1}\right)\theta \Rightarrow \mathcal{N}\left(0, \sigma^2 \left(\gamma^2\| \left(I_d - \gamma A^\top \right)^{-1} \theta\|^2 +1\right) \left(I_d - \gamma A^\top\right)^{-1} P_\infty^{-1} \left(I_d - \gamma A\right)^{-1}  \right)
    \end{align*}
\end{proposition}

\begin{corollary}
    \label{cor:MB_estim}
    As $n \to \infty$,
    \begin{equation*}
        n \EE{\|\hat{\beta}_{\rm LSTD}  - \beta \|^2 } \to \sigma^2 \left(\gamma^2\| \left(I_d - \gamma A^\top \right)^{-1} \theta\|^2 +1\right) \left\| P_\infty^{-1/2} \left(I_d - \gamma A\right)^{-1} \right\|^2 
    \end{equation*}
\end{corollary}

\begin{proof}
We define $
    M = \begin{bmatrix}
A \\
\theta^\top
\end{bmatrix}
$ and state an asymptotic result on the least-square regressor of $M$ borrowed from \cite{tu2019gap} (Lemma A.1).

\begin{lemma}   
    \label{lem:dynamicsCLT}
    Let \begin{equation*}
    \hat{M} = \begin{bmatrix}
\hat{A} \\
\hat{\theta}^\top
\end{bmatrix} = \arg\min_{M \in \mathbb{R}^{(d+1)\times d}} \sum_{t=0}^{n-1}\left\|\begin{bmatrix}
X^{(t+1)} \\
R^{(t)}
\end{bmatrix} - MX^{(t)}\right\|^2.
\end{equation*}
Then $\hat{M} \xrightarrow{\text{a.s.}} M$ and $ \sqrt{n} \vect{\hat{M} - M} \Rightarrow \mathcal{N}\left(0, \sigma^2 P_\infty^{-1} \otimes I_{d+1} \right)$
    In particular, $\hat{A}  \xrightarrow{\text{a.s.}} A$, $\hat{\theta}  \xrightarrow{\text{a.s.}} \theta$ and
    \begin{align*}
        \sqrt{n} \vect{\hat{A} - A} \Rightarrow \mathcal{N}\left(0, \sigma^2 P_\infty^{-1} \otimes I_d \right); \quad \quad
        \sqrt{n} \vect{\hat{\theta} - \theta} \Rightarrow \mathcal{N}\left(0, \sigma^2 P_\infty^{-1}\right).   
    \end{align*}
\end{lemma}
Equipped with this lemma, we prove a central limit theorem on the matrix $\left(I_d - \gamma \hat{A}^\top\right)^{-1}$.
\begin{lemma}
    \label{lem:CLT_delta}
    As $n \to \infty$,
    \begin{align*}
    \sqrt{n}\left(\left(I_d - \gamma \hat{A}^\top\right)^{-1} - \left(I_d - \gamma A^\top\right)^{-1}\right) \Rightarrow \gamma \left(I_d-\gamma A^\top\right)^{-1} N \left(I_d-\gamma A^\top\right)^{-1}
\end{align*}
where $\vect{N^\top} \sim \mathcal{N}\left(0, \sigma^2 P_\infty^{-1} \otimes I_d \right)$. 
\end{lemma}

\begin{proof}
This is an application of the delta method: we define the function $f(X) = (I_d - \gamma X)^{-1}$. Its derivative at $X$ is $df(X,H) = \gamma (I_d - \gamma X)^{-1}H(I_d - \gamma X)^{-1}$. Hence 
\begin{align*}
    \sqrt{n}\left(\left(I_d - \gamma \hat{A}^\top\right)^{-1} - \left(I_d - \gamma A^\top\right)^{-1}\right) \Rightarrow \gamma \left(I_d-\gamma A^\top\right)^{-1} N \left(I_d-\gamma A^\top\right)^{-1}
\end{align*}

where Lemma \ref{lem:dynamicsCLT} ensures that $\vect{N^\top} \sim \mathcal{N}\left(0, \sigma^2 P_\infty^{-1} \otimes I_d \right)$. 
\end{proof}

To obtain the final result, we decompose the quantity of interest as follow.
\begin{align*}
    \sqrt{n}\left(\left(I_d - \gamma \hat{A}^\top\right)^{-1} \hat{\theta} - \left(I_d - \gamma A\right)^{-1} \theta \right) =&  \sqrt{n}\left(\left(I_d - \gamma \hat{A}^\top\right)^{-1} - \left(I - \gamma A^\top\right)^{-1}\right)  \left(\hat{\theta} - \theta\right)\\ 
    &+ \sqrt{n} \left(\left(I_d - \gamma \hat{A}^\top \right)^{-1} - \left(I_d - \gamma A^\top \right)^{-1}\right)\theta \\ 
    &+ \sqrt{n}\left(I_d - \gamma A^\top \right)^{-1} \left(\hat{\theta} - \theta\right).
\end{align*}
Assembling Lemma \ref{lem:dynamicsCLT} and \ref{lem:CLT_delta}, we can analyse all three parts:
\begin{enumerate}
    \item The first part is a second order error term. From Slutsky's theorem:
    \begin{equation*}
        \sqrt{n}\left(\left(I_d - \gamma \hat{A}^\top\right)^{-1} - \left(I - \gamma A^\top\right)^{-1}\right)  \left(\hat{\theta} - \theta\right) \Rightarrow 0
    \end{equation*}
    \item The second part converges to a normal distribution, as per Lemma \ref{lem:CLT_delta}
    \begin{equation*}
        \sqrt{n} \left(\left(I_d - \gamma \hat{A}^\top \right)^{-1} - \left(I_d - \gamma A^\top \right)^{-1}\right)\theta \Rightarrow \mathcal{N}\left(0, \gamma^2\sigma^2 \| \left(I_d - \gamma A^\top \right)^{-1} \theta\|^2 \left(I_d - \gamma A^\top\right)^{-1} P_\infty^{-1} \left(I_d - \gamma A\right)^{-1}  \right)
    \end{equation*}
    where we used the Lemma \ref{lem:rdm_mtrx} to simplify the covariance matrix. 
    \item The third part converges to a normal distribution, as per Lemma \ref{lem:dynamicsCLT}
    \begin{equation*}
        \sqrt{n}\left(I_d - \gamma A^\top \right)^{-1} \left(\hat{\theta} - \theta\right) \Rightarrow \mathcal{N}\left(0, \sigma^2 \left(I_d - \gamma A^\top \right)^{-1} P_\infty^{-1} \left(I_d - \gamma A \right)^{-1}  \right).
    \end{equation*}
\end{enumerate}

Finally, Lemma \ref{lem:dynamicsCLT} implies that the 2. and 3. converge jointly to a normal distribution. The sum also converges to a normal distribution where the covariance is the sum of the variance (given the asymptotic independence of 2 and 3), concluding the proof.

\end{proof}

\subsection{Central Limit Theorem in the diagonal linear dynamic setting}
\begin{proposition}
    The model based estimator when the dynamics are known to be diagonal verifies
    \begin{equation*}
        \sqrt{n}\left(\hat{\beta}_{\rm diag} - \beta\right) \Rightarrow \mathcal{N}(0, \Sigma)
    \end{equation*}
    where $\Sigma$ is diagonal and 
    $$\Sigma_{i,i} = \left(\frac{\gamma^2\theta_i^2}{(1 - \gamma A_{i,i})^2} + 1\right) \frac{1 - A_{i,i}^2}{(1 - \gamma A_{i,i})^2}$$
\end{proposition}

\begin{corollary} \label{cor:diag_MSE}As $n \to \infty$,
    \begin{equation*}
    n \EE{\|\hat{\beta}_{\rm diag}- \beta \|^2} \to \sum _{i=1}^d \left(\frac{\gamma^2\theta_i^2}{(1 - \gamma A_{i,i})^2} + 1\right) \frac{1 - A_{i,i}^2}{(1 - \gamma A_{i,i})^2}
\end{equation*}
\end{corollary}

When $A$ is known to be diagonal, the model based estimator can be written $\hat{\beta}_{{\rm diag},i} = \dfrac{\hat{\theta}_i}{1 - \gamma \hat{A}_{i,i}}$. Let $a = (A_{i,i})_{i = 1,\dots, d}$. Similarly to Lemma \ref{lem:dynamicsCLT}, we have in this case:

\begin{equation*}
\sqrt{n} \left(\begin{bmatrix}
\hat{a}^\top \\
\hat{\theta}^\top
\end{bmatrix} - \begin{bmatrix}
a^\top \\
\theta^\top
\end{bmatrix} \right) \Rightarrow \mathcal{N}\left(0, \begin{bmatrix}
P_\infty^{-1} \quad 0\\
\quad0 \quad P_\infty^{-1}
\end{bmatrix}  \right)
\end{equation*}

Note that when $A$ is diagonal, so is $P_\infty$ and $P_{\infty, i,i} = \dfrac{\sigma^2}{1-A_{i,i}^2}$. Hence, we can treat each of the coordinate of $\hat{\beta}_{{\rm diag},i}$ separately. We then have 
\begin{equation*}
    \sqrt{n} \left(\hat{\beta}_{{\rm diag},i} - \beta_i\right) \Rightarrow \mathcal{N}\left(0,\sigma^2  \left(\frac{\gamma^2\theta_i^2}{(1 - \gamma A_{i,i})^2} + 1\right) \frac{1 - A_{i,i}^2}{\sigma^2(1 - \gamma A_{i,i})^2} \right)
\end{equation*}
Using that each component of $\hat{\beta}_{\rm diag}$ is asymptotically independent of the others, we get the result.

\section{Technical lemmas on random matrices}

\begin{lemma}
    \label{lem:rdm_mtrx}
    Let $M$ be a random matrix such that $\Cov{\vect{M^\top}} = B \otimes C$ and $\lambda \in \mathbb{R}^d$ be a deterministic vector, then 
    \begin{equation*}
        \Cov{M\theta} = (\theta^\top C \theta) B
    \end{equation*}
\end{lemma}

\begin{proof}
    Let $i,j \in \{1, \dots, d\}$,
    \begin{align*}
        \Cov{(M\theta)_i, (M\theta)_j} &= \Cov{\sum_{k =1}^d M_{i,k}\theta_k, \sum_{l = 1}^d M_{j,l} \theta_l } = \sum_{k =1}^d \sum_{l = 1}^d \Cov{ M_{i,k}\theta_k,  M_{j,l} \theta_l }\\
        &= \sum_{k =1}^d \sum_{l = 1}^d  \Cov{ \vect{M^\top}_{n(i-1)+k},  \vect{M^\top}_{n(j-1)+l}} \theta_k \theta_l\\
        &=  \sum_{k =1}^d \sum_{l = 1}^d  B_{i,j} C_{k,l} \theta_k \theta_l = B_{i,j} \theta^\top C \theta.
    \end{align*}
\end{proof}
\section{Equivalence between LSTD and a model-based method for LQR}
\label{sec:LSTD_eq_LQR}
As described in Section \ref{sec:LQR}, we can consider without any loss of generality that the state is $Z^{(t)} = \vect{X^{(t)}(X^{(t)})^\top}$. Doing this change of variable, the dynamics are 
\begin{align*}
    \EE{Z^{(t+1)} | Z^{(t)}} &= M Z^{(t)} + \vect{I_d}; \quad \quad \text{where }  M = A \otimes A\\
    \EE{R^{(t)} | Z^{(t)}} &=  \vect{Q}^\top Z^{(t)}
\end{align*}

Let $\mathfrak{M}$ be the set of MRPs with state space $\mathcal{S} = \mathbb{R}^{d^2}$ such that there exists $M \in \mathbb{R}^{d^2 \times d^2}, \theta \in \mathbb{R}^{d^2}$ such that
\begin{align*}
    \EE{X^{(t+1)} | X^{(t)}} = M X^{(t)} + \vect{I_d}; \quad \quad \EE{R^{(t)} | X^{(t)}} =  \theta^\top X^{(t)}.
\end{align*}
Let $\mathbb{M}_{\rm LQR}$ be the class of all LQR instances where the state space is viewed in $\mathbb{R}^{d^2}$. It is clear that $\mathbb{M}_{\rm LQR} \subset \mathfrak{M}$. 

\begin{proposition}[Loss of information in LQR]
    The class of models $\mathfrak{M}$ have the same class of value function as $\mathbb{M}_{\rm LQR}$:
    \begin{equation*}
        \{V_\mathcal{M} | \mathcal{M} \in \mathfrak{M} \} = \{V_\mathcal{M} | \mathcal{M} \in \mathbb{M}_{\rm LQR} \}
    \end{equation*}
\end{proposition}

The proof is immediate from Proposition \ref{prop:LQR_V_exact} and verifying that $V_\mathcal{M}$ lives in the same set for $\mathcal{M} \in \mathfrak{M}$.

We now describe the model-based procedure to which LSTD is equivalent in this setting. First, we recall the LSTD estimator for LQR:

\begin{equation*}
    \hat{P}_{\rm LSTD} = \left(\Phi(\Phi - \gamma \Phi')^\top\right)^{-1}\Phi R
\end{equation*}

Where $\Phi \in \mathbb{R}^{d^2 \times n}, \Phi' \in \mathbb{R}^{d^2 \times n}$ are the matrix where the $t$-th column is respectively $Z^{(t)} + \frac{\gamma}{1-\gamma}\vect{I_d}$ and $Z^{(t+1)}+ \frac{\gamma}{1-\gamma}\vect{I_d}$ and $R$ is the vector $(R^{(t)})_t$. Similarly to the case with linear reward, this can be re-written as 
\begin{align*}
    \hat{P}_{\rm LSTD} &= \left(\Phi\Phi^\top - \gamma \Phi(\Phi')^\top\right)^{-1}\Phi R\\
     &= \left(I_{d^2} - \gamma \left(\Phi\Phi^\top\right)^{-1}\Phi(\Phi')^\top\right)^{-1}\left(\Phi\Phi^\top\right)^{-1}\Phi R \\
    &= \left(I_{d^2} - \gamma \hat{M}^\top \right)^{-1} \hat{\theta}
\end{align*}

Where 
\begin{align*}
    \hat{M} &= \arg \min_{M \in \mathbb{R}^{d^2 \times d^2}} \sum_{t=1}^n \left\| \left(Z^{(t+1)} + \frac{\gamma}{1-\gamma}\vect{I_d}\right) - M \left(Z^{(t)}  +\frac{\gamma}{1-\gamma}\vect{I_d} \right) \right\|^2 \\
    \hat{\theta} &= \arg \min_{\theta \in \mathbb{R}^{d^2}} \sum_{t=1}^n \left\| R^{(t)} - \theta^\top \left(Z^{(t)}  +\frac{\gamma}{1-\gamma}\vect{I_d} \right) \right\|^2
\end{align*}

In words, $\hat{M}$ is the best linear approximation to the mapping from $Z^{(t)}  +\frac{\gamma}{1-\gamma}\vect{I_d} $ to $Z^{(t+1)}  +\frac{\gamma}{1-\gamma}\vect{I_d} $. While the true mapping isn't linear (it is affine), the best linear approximation to it, under infinite data, is the mapping $z \mapsto M z$ where $M$ is the dynamic matrix. Therefore, this model estimation procedure asymptotically recovers the true dynamics $M$. Similarly, it recovers the true reward parameter $\theta$. However, the loss of information materializes in that the best linear approximation $\hat{M}$ is learned over the set of all $d^2 \times d^2$ matrices instead of only considering the ones of interest, that is those that can be written $A \otimes A$ for some matrix $A \in \mathbb{R}^{d \times d}$.

\end{document}